\documentclass[conference]{IEEEtran}
\IEEEoverridecommandlockouts
\usepackage{cite}
\usepackage{amsmath,amssymb,amsfonts,amsthm}
\newtheorem{proposition}{Proposition}
\usepackage{algorithm}
\usepackage{algorithmic}
\usepackage{graphicx}
\usepackage{textcomp}
\usepackage{xcolor}
\usepackage{booktabs}
\usepackage{float}
\def\BibTeX{{\rm B\kern-.05em{\sc i\kern-.025em b}\kern-.08em
    T\kern-.1667em\lower.7ex\hbox{E}\kern-.125emX}}
\begin{document}

\title{CacheMPC: Certified Cached Model Predictive Control for Quadruped Locomotion}

\author{\IEEEauthorblockN{Nimesh Khandelwal}
\IEEEauthorblockA{\textit{Dept. of Mech. Engg.} \\
\textit{IIT Kanpur}\\
Kanpur, India \\
nimesh20@iitk.ac.in}
\and
\IEEEauthorblockN{Mehul Anand}
\IEEEauthorblockA{\textit{Dept. of Mat. Sci.} \\
\textit{IIT Roorkee}\\
Roorkee, Uttarakhand \\
mehul\_a@mt.iitr.ac.in}
\and
\IEEEauthorblockN{Shakti S. Gupta}
\IEEEauthorblockA{\textit{Dept. of Mech. Engg.} \\
\textit{IIT Kanpur}\\
Kanpur, India \\
ssgupta@iitk.ac.in}
\and
\IEEEauthorblockN{Mangal Kothari\thanks{M. Kothari was with the Department of Aerospace Engineering, Indian Institute of Technology Kanpur, Kanpur, India. He is currently with ADASI, EDGE Group, Abu Dhabi, UAE. This work was performed while he was with the Indian Institute of Technology Kanpur.}}
\IEEEauthorblockA{\textit{Dept. of Aerospace Engg.} \\
\textit{IIT Kanpur}\\
Kanpur, India \\
mangalgnc@gmail.com}
}

\maketitle

\begin{abstract}
Model Predictive Control (MPC) is the standard predictive layer in hierarchical quadruped controllers, but the per-cycle QP solve limits the update rate achievable on embedded processors. Because legged gaits revisit a bounded region of state space, MPC solutions admit caching and reuse. This paper proposes \emph{Certified CacheMPC}: a Locality-Sensitive-Hashed cache of horizon contact-force trajectories, partitioned by contact mode, retrieved at query time and accepted only when an a-posteriori per-query certificate confirms primal feasibility and a Lagrangian dual-gap upper bound on cost suboptimality. A bounded-budget controller schedule combines top-$K$ certified retrieval, a deadline-bounded QP solve, and a shifted last-certified fallback. The framework is evaluated on a Unitree Go2 across $2{,}038$ usable cold-controller MuJoCo trials, including a $600$-trial $n\!=\!50$ campaign at three failure-boundary cells, and a first-deploy session on the on-robot NVIDIA Orin NX. The un-gated cache delivers a $25\times$ median solve-time speedup in simulation and an $18.7\times$ median speedup on hardware. At $n\!=\!50$ no statistically significant difference in closed-loop stable rate is detected between the cache variants and the no-cache baseline at any tested cell. The certificate's contribution to closed-loop safety is not resolvable at the present sample size.
\end{abstract}


\section{Introduction}
\label{sec:intro}

Hierarchical control architectures combining a low-rate predictive controller with a high-rate reactive controller have emerged as the state of the art for model-based legged locomotion~\cite{convexmpc,mit_minicheetah,sleiman2021,bledt2018cheetah3,hutter2016anymal,kim2019wbic}, complementing learning-based controllers that have demonstrated robust locomotion at scale~\cite{hwangbo2019,margolis2024}. The predictive controller, typically a convex MPC on a simplified dynamics model, optimises contact forces over a finite horizon; the reactive controller, typically a Whole-Body Controller (WBC)~\cite{sentis2005wbc,wwbc_bellicoso}, uses the full rigid-body dynamics to produce physically consistent joint torques at high frequency. The quality of the MPC solution directly affects tracking and robustness, but the per-cycle solve cost limits the achievable MPC rate on embedded processors such as the NVIDIA Orin NX shipped with the Unitree Go2: a state-dependent contact horizon, larger problem size, or sequential preprocessing all push the per-tick MPC update into the millisecond range, leaving little headroom inside a typical $5$\,ms control budget. The MPC is therefore the limiting layer for the closed-loop update rate.

Several lines of work address this gap. Real-time iteration NMPC~\cite{diehl2002rti,verschueren2022acados} amortises a fresh nonlinear solve over multiple ticks but still requires solver progress every cycle. TinyMPC~\cite{tinympc} pre-computes ADMM solver factorisations but requires linearisable dynamics. Neural function approximation~\cite{mpcnet} distils MPC solutions into networks but requires separate training and lacks runtime correctness guarantees; recent work on safe approximate MPC~\cite{hertneck2018lcss,paulson2020} adds verification but operates offline. Explicit MPC~\cite{explicitmpc,tondel2003mpqp} pre-computes the solution as a piecewise-affine function of the state, but the polyhedral region count scales poorly with problem dimension. Sampling-based predictive controllers such as MPPI~\cite{williams2017mppi} sidestep the QP entirely at the cost of high parallel compute. The idea of caching and reusing MPC solutions appears in parametric settings~\cite{param_adaptive_mpc}, but without per-query safety guarantees.

This paper proposes Certified CacheMPC. A Locality-Sensitive-Hashed (LSH) cache, partitioned by contact mode so that the QP structural pattern is shared within a partition, returns previously computed contact-force trajectories at query time. An a-posteriori per-query certificate, evaluated against the current QP's own KKT data, produces a Lagrangian dual-gap quantity $\Gamma(\bar{\mathbf{U}})$ that upper-bounds the true suboptimality $\Delta J=J(\bar{\mathbf{U}})-J^\star$ of any accepted retrieval, independent of where the proposal originated. A bounded-budget controller schedule combines top-$K$ certified retrieval, a deadline-bounded QP solve, and a shifted last-certified fallback.

The contributions of this paper are:
\begin{enumerate}
\item A deterministic per-query QP certificate (Proposition~\ref{prop:cert}) whose Lagrangian dual-gap quantity $\Gamma(\bar{\mathbf{U}})$ upper-bounds the true suboptimality of any feasible cached proposal independent of its provenance.
\item A bounded-budget controller schedule (Section~\ref{sec:bounded_time_contract}) combining top-$K$ certified retrieval, a deadline-bounded solver, and a shifted last-certified fallback, exercised empirically under deliberately tight deadlines in Section~\ref{sec:bounded_time_demo}.
\item A simulation campaign of $2{,}038$ usable cold-controller trials and a first-deploy hardware session on the Unitree Go2 with the on-robot NVIDIA Orin NX (Sections~\ref{sec:experiments}--\ref{sec:hardware}), characterising the speed, tracking, and failure-boundary safety of four cache configurations against the no-cache baseline.
\end{enumerate}

The remainder of the paper is organised as follows. Section~\ref{sec:related} reviews related work; Section~\ref{sec:background} states the convex MPC and WBC formulations; Section~\ref{sec:cachempc} describes the certified framework and its propositions; Section~\ref{sec:experiments} reports the simulation evaluation; Section~\ref{sec:hardware} reports the hardware session; Section~\ref{sec:threats} discusses threats to validity; Section~\ref{sec:conclusion} concludes; Section~\ref{sec:future_work} lists future-work directions.

\section{Related Work}
\label{sec:related}

\subsection{MPC for Quadruped Locomotion}
Convex MPC formulations based on the single rigid body dynamics (SRBD) model have been widely adopted for quadruped control~\cite{convexmpc,mit_minicheetah,bledt2018cheetah3}. These formulations linearise the rotational dynamics and approximate friction cones as pyramids to cast the problem as a QP, enabling real-time solutions using solvers such as qpOASES~\cite{qpoases} and HPIPM~\cite{hpipm}. Torque-controlled feedback MPC variants~\cite{grandia2019feedback} and whole-body trajectory optimisation frameworks~\cite{sleiman2021,kim2019wbic} extend the formulation to richer dynamics at the cost of a larger online problem.

\subsection{Approaches to Accelerate MPC}
Explicit MPC \cite{explicitmpc} pre-computes the optimal control law as a piecewise affine function of the state by solving a multi-parametric QP offline. The online evaluation reduces to a point-location problem in a polyhedral partition of the state space. While effective for low-dimensional systems, the number of regions grows exponentially with problem dimension, making the full enumeration impractical for the 13-state MPC problems common in quadruped control. Our framework can be viewed as an online, sample-based variant of explicit MPC: each cache entry is a local affine piece of the explicit law, populated lazily from the same QPs that the controller would have solved anyway, and accepted at query time only when an a-posteriori certificate vouches for it.

Function approximation methods learn a mapping from state to MPC solution using supervised learning. MPCNET \cite{mpcnet} trains a neural network on MPC demonstrations and uses it for online inference. Similarly, iterative learning control (ILC) approaches \cite{ilc_ref} maintain a library of control actions indexed by trajectory phase, achieving faster execution through lookup rather than optimization. These methods require offline training or separate optimization stages and are difficult to adapt online to changing environments.

TinyMPC \cite{tinympc} achieves speedup by pre-computing and caching matrix factorizations within an ADMM solver, avoiding expensive operations at runtime. While effective for convex problems with fixed structure, it does not address the broader question of reusing full MPC solutions across similar states.

Warm-starting is a widely used technique where the solution from the previous time step initializes the solver for the current problem. While effective for reducing iteration counts, warm-starting provides diminishing returns when the state changes significantly between MPC updates, and still requires running the solver to completion.

\subsection{Cache-Based and Lookup Approaches}
The idea of caching and reusing optimization solutions has been explored in various control contexts. In \cite{param_adaptive_mpc}, a parameter-adaptive framework stores MPC solutions indexed by system parameters for fast retrieval. Our approach differs in that we cache solutions indexed by the full operational state using LSH, dynamically update the cache during operation, and exploit the contact-schedule structure specific to legged locomotion. To the best of our knowledge, CacheMPC is the first framework to combine LSH-based caching with contact-schedule partitioning and a per-query optimality certificate for quadruped MPC.

\section{Background}
\label{sec:background}

\subsection{Convex MPC for Quadruped Locomotion}
\label{sec:mpc_formulation}

We use a convex MPC formulation based on the SRBD model, following \cite{convexmpc}. The robot is modeled as a single rigid body with mass $m$ and rotational inertia $\mathbf{I}_b \in \mathbb{R}^{3\times 3}$, actuated by contact forces at the feet. The state vector is defined as
\begin{equation}
    \mathbf{x} = \begin{bmatrix} \boldsymbol{\Theta}^T & \mathbf{p}^T & \boldsymbol{\omega}^T & \dot{\mathbf{p}}^T & g \end{bmatrix}^T \in \mathbb{R}^{13},
    \label{eq:state}
\end{equation}
where $\boldsymbol{\Theta} \in \mathbb{R}^3$ is the ZYX Euler angle vector representing the body orientation, $\mathbf{p} \in \mathbb{R}^3$ is the position of the center of mass (CoM), $\boldsymbol{\omega} \in \mathbb{R}^3$ is the angular velocity in the world frame, $\dot{\mathbf{p}} \in \mathbb{R}^3$ is the CoM linear velocity, and $g$ is the gravitational constant (appended for notational convenience). The control input is the stacked vector of contact forces:
\begin{equation}
    \mathbf{u} = \begin{bmatrix} \mathbf{f}_1^T & \mathbf{f}_2^T & \mathbf{f}_3^T & \mathbf{f}_4^T \end{bmatrix}^T \in \mathbb{R}^{12},
    \label{eq:control}
\end{equation}
where $\mathbf{f}_i \in \mathbb{R}^3$ is the ground reaction force at foot $i$.

The continuous-time dynamics are linearized about the current state to obtain $\dot{\mathbf{x}} = \mathbf{A}_c \mathbf{x} + \mathbf{B}_c \mathbf{u}$, where
\begin{equation}
    \mathbf{A}_c = \begin{bmatrix}
        \mathbf{0}_3 & \mathbf{0}_3 & \mathbf{R}_z(\psi) & \mathbf{0}_3 & \mathbf{0}_{3\times 1}\\
        \mathbf{0}_3 & \mathbf{0}_3 & \mathbf{0}_3 & \mathbf{I}_3 & \mathbf{0}_{3\times 1}\\
        \mathbf{0}_3 & \mathbf{0}_3 & \mathbf{0}_3 & \mathbf{0}_3 & \mathbf{0}_{3\times 1}\\
        \mathbf{0}_3 & \mathbf{0}_3 & \mathbf{0}_3 & \mathbf{0}_3 & \hat{\mathbf{g}}\\
        \mathbf{0}_{1\times 3} & \mathbf{0}_{1\times 3} & \mathbf{0}_{1\times 3} & \mathbf{0}_{1\times 3} & 0
    \end{bmatrix},
\end{equation}
with $\mathbf{R}_z(\psi)$ being the yaw rotation matrix and $\hat{\mathbf{g}} = \begin{bmatrix} 0 & 0 & -1\end{bmatrix}^T$. The input matrix $\mathbf{B}_c$ encodes the mapping from contact forces to CoM accelerations and angular accelerations through the moment arms of each foot relative to the CoM. The dynamics are discretized with timestep $\Delta t$ using a zero-order hold to obtain $\mathbf{x}_{k+1} = \mathbf{A}\mathbf{x}_k + \mathbf{B}\mathbf{u}_k$.

The MPC problem is formulated over a horizon of $N$ steps as:
\begin{equation}
\begin{aligned}
    \min_{\mathbf{u}_0, \ldots, \mathbf{u}_{N-1}} \quad & \sum_{k=0}^{N-1} \left[ (\mathbf{x}_k - \mathbf{x}_k^{\text{ref}})^T \mathbf{Q} (\mathbf{x}_k - \mathbf{x}_k^{\text{ref}}) + \mathbf{u}_k^T \mathbf{R} \mathbf{u}_k \right] \\
    \text{s.t.} \quad & \mathbf{x}_{k+1} = \mathbf{A}\mathbf{x}_k + \mathbf{B}\mathbf{u}_k, \\
    & |f_{i,x}| \leq \mu f_{i,z}, \quad |f_{i,y}| \leq \mu f_{i,z}, \\
    & f_{z,\min} \leq f_{i,z} \leq f_{z,\max}, \\
    & f_{i,z} = 0 \quad \text{if foot } i \notin \mathcal{C}_k,
\end{aligned}
\label{eq:mpc_qp}
\end{equation}
where $\mathbf{Q} \succeq 0$ and $\mathbf{R} \succ 0$ are the state and control weight matrices, $\mu$ is the friction coefficient, and $\mathcal{C}_k$ is the set of feet in contact at step $k$, determined by the gait schedule. The friction cone is linearized as a pyramid with four faces. The resulting problem is the strictly convex QP \eqref{eq:cert_qp} and is solved with HPIPM~\cite{hpipm} in production; the formulation is solver-agnostic and the certificate of Section~\ref{sec:cachempc} consumes only the assembled condensed QP data.

\paragraph{Fixed-contact horizon.}
The per-stage contact set $\mathcal{C}_k$ is held at the current measured contact set $\mathcal{C}_0$ for every $k\!\in\![0,N\!-\!1]$. The condensed QP \eqref{eq:cert_qp} therefore reuses one $(\mathbf{B}_k,\underline{\mathbf{d}},\overline{\mathbf{d}})$ block at every stage, partitioning by the four-bit contact mask $\boldsymbol{\sigma}$ is exact for the QP structure, and contact-switch anticipation is delegated to the gait scheduler upstream of the MPC. A horizon-anticipative variant of \cite{convexmpc} that consumes a gait-table-driven contact sequence would re-key the cache on the full $(\mathcal{C}_0,\mathcal{C}_1,\dots,\mathcal{C}_{N-1})$ sequence rather than $\boldsymbol{\sigma}\!=\!\mathcal{C}_0$; the certificate of Section~\ref{sec:cachempc} extends unchanged to that variant because it operates on the assembled condensed QP and does not depend on the source of $\boldsymbol{\sigma}$.

\subsection{KinoDynamic Whole Body Control}
\label{sec:wbc}

The reactive layer uses a KinoDynamic WBC that operates at the proprioceptive estimator rate (250~Hz on our build) and therefore at a higher frequency than the MPC. It consists of two cascaded stages: a Kinematic WBC (KinWBC) followed by a Weighted Whole Body Controller (W-WBC), forming the complete MPC~+~KinWBC~+~W-WBC control stack.

The KinWBC modifies the reference joint positions to better satisfy task-space objectives. Given the current robot configuration and the desired operational-space motions (base pose, foot positions), it computes updated joint position references using prioritized inverse kinematics, following \cite{mit_minicheetah}. This kinematic adjustment ensures that the joint references are consistent with the desired base and foot motions before being passed to the dynamic layer.

The W-WBC uses the full rigid-body dynamics (RBD) of the quadruped to compute joint torques. We adopt the weighted WBC formulation \cite{wwbc_bellicoso} that minimizes a weighted sum of task errors subject to the equation of motion:
\begin{equation}
    \mathbf{M}(\mathbf{q})\ddot{\mathbf{q}} + \boldsymbol{\eta}(\mathbf{q}, \dot{\mathbf{q}}) = \mathbf{S}^T\boldsymbol{\tau} + \mathbf{J}_c^T\mathbf{F}_c,
    \label{eq:eom}
\end{equation}
where $\mathbf{q} \in \mathbb{R}^{18}$ is the generalized coordinate vector, $\mathbf{M}$ is the inertia matrix, $\boldsymbol{\eta}$ represents Coriolis and gravitational terms, $\mathbf{S}$ is the selection matrix for actuated joints, $\boldsymbol{\tau} \in \mathbb{R}^{12}$ are joint torques, and $\mathbf{F}_c$ are the contact forces.

The WBC optimizes joint accelerations $\ddot{\mathbf{q}}$, joint torques $\boldsymbol{\tau}$, and contact forces $\mathbf{F}_c$ simultaneously:
\begin{equation}
\begin{aligned}
    \min_{\ddot{\mathbf{q}}, \boldsymbol{\tau}, \mathbf{F}_c} \;\; & \sum_{i} w_i \| \mathbf{J}_i \ddot{\mathbf{q}} + \dot{\mathbf{J}}_i \dot{\mathbf{q}} - \ddot{\mathbf{x}}_i^{\text{cmd}} \|^2 \\
    & \quad + w_\tau \|\boldsymbol{\tau}\|^2 + w_f \|\mathbf{F}_c - \mathbf{F}_c^{\text{ref}}\|^2\\
    \text{s.t.} \;\; & \text{Eq. } (\ref{eq:eom}), \quad \text{contact constraints,}
\end{aligned}
\label{eq:wbc}
\end{equation}
where $\mathbf{J}_i$ and $\ddot{\mathbf{x}}_i^{\text{cmd}}$ are the Jacobian and commanded acceleration for task $i$ (computed using the KinWBC-adjusted references), and $\mathbf{F}_c^{\text{ref}}$ is the reference contact force from the MPC. The tracking term $\|\mathbf{F}_c - \mathbf{F}_c^{\text{ref}}\|^2$ is the interface through which the MPC solution---cached or freshly solved---enters the W-WBC. The W-WBC enforces dynamic consistency with the full rigid-body dynamics and the friction and unilaterality constraints that the SRBD model abstracts away, while the KinWBC reconciles joint references with the realised task-space motion. The WBC is the final feasibility filter for the full-body dynamics; the per-query certificate of Section~\ref{sec:cachempc} bounds the SRBD-objective gap of any accepted reference by $\beta(\bar{\mathbf{U}})$, leaving model mismatch and full-body feasibility to the WBC.

\section{Certified Memory-Augmented MPC}
\label{sec:cachempc}

A retrieved control is accepted only when an a-posteriori certificate, evaluated against the current predictive problem's own optimality conditions, confirms both primal feasibility and a budgeted suboptimality bound. The certificate elevates a cache entry from a heuristic neighbour in feature space to a verified local piece of the explicit-MPC law. Because the optimal-control map is piecewise affine with discontinuities at active-set boundaries, a near neighbour in feature space can lie in a different critical region; the per-query certificate is what rules this out. Fig.~\ref{fig:architecture} shows the resulting architecture.

\begin{figure*}[t]
    \centering
    \includegraphics[width=\textwidth]{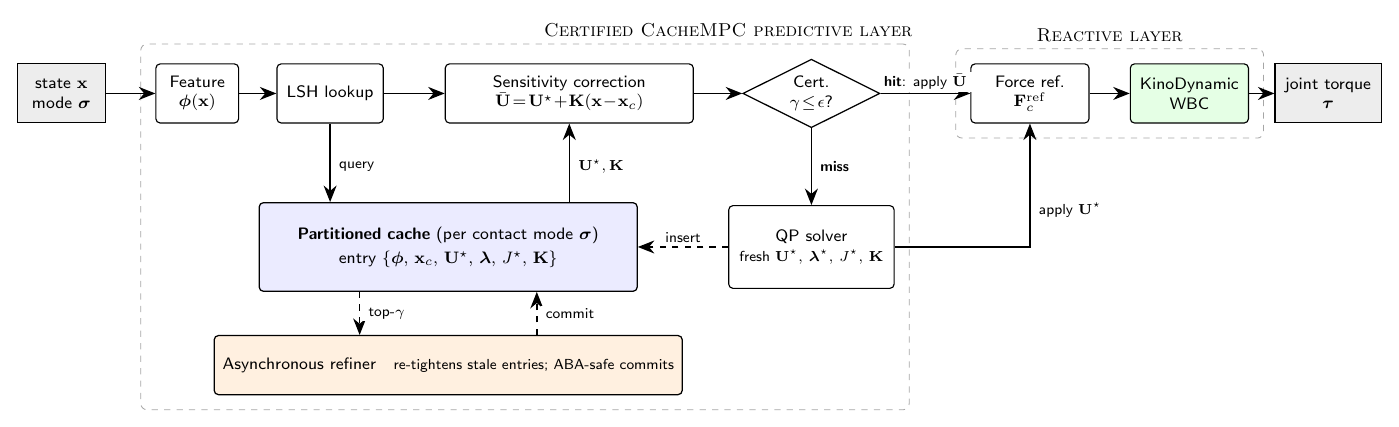}
    \caption{Certified CacheMPC architecture. The predictive layer hashes $\boldsymbol\phi(\mathbf{x})$ into the contact-mode-partitioned cache and accepts a sensitivity-corrected candidate $\bar{\mathbf{U}}$ only when $\rho_{\mathrm{feas}}$ and the dual-gap $\Gamma(\bar{\mathbf{U}})$ satisfy the budget $\beta(\bar{\mathbf{U}})$; on rejection the QP is solved and the fresh tuple is inserted. An asynchronous refiner re-tightens stale entries, and the accepted $\mathbf{F}_c^{\mathrm{ref}}$ is consumed by the KinoDynamic WBC.}
    \label{fig:architecture}
\end{figure*}

\subsection{Feature Vector and Candidate Generation via LSH}
\label{sec:feature_lsh}

Each cache entry is indexed by a feature vector $\boldsymbol{\phi}\in\mathbb{R}^{18}$ that compresses the operationally relevant state:
\begin{equation}
    \boldsymbol{\phi} = \begin{bmatrix} \dot{\mathbf{p}}^T & \dot{\mathbf{p}}_{\text{des}}^T & ({}^{\mathcal{B}}\mathbf{r}_{p_1})^T & \cdots & ({}^{\mathcal{B}}\mathbf{r}_{p_4})^T \end{bmatrix}^T,
    \label{eq:feature}
\end{equation}
combining the measured and desired CoM velocities and the four body-frame foot positions. The cache is partitioned by contact mode $\boldsymbol{\sigma}$ so that the QP structure (constraint shape, active swing/stance bits) is shared within a partition; a trotting gait, for example, populates only two partitions. Approximate nearest-neighbour search inside each partition uses Locality-Sensitive Hashing~\cite{lsh_original,datar2004lsh,andoni2008lsh} with the Euclidean $p$-stable construction of~\cite{datar2004lsh}: $L$ tables of $k$ hash functions $h_{\mathbf{a},b}(\boldsymbol{\phi})=\lfloor(\mathbf{a}\!\cdot\!\boldsymbol{\phi}+b)/w\rfloor$, $\mathbf{a}\sim\mathcal{N}(\mathbf{0},\mathbf{I})$, $b\sim\mathcal{U}[0,w]$, giving sub-linear query time in 18-dimensional space. LSH serves strictly as a candidate generator; acceptance is decided by the certificate of Section~\ref{sec:cert}.

\subsection{Condensed Problem and Cache Entry Schema}
\label{sec:cert_qp}

After eliminating the predicted states via the linearised dynamics of Section~\ref{sec:mpc_formulation}, the convex MPC at state $\mathbf{x}$ and contact mode $\boldsymbol{\sigma}$ is the strictly convex QP
\begin{equation}
    \min_{\mathbf{U}}\;\tfrac{1}{2}\mathbf{U}^{\!\top}\mathbf{H}\,\mathbf{U} + \mathbf{g}^{\!\top}\mathbf{U}
    \quad\text{s.t.}\quad
    \underline{\mathbf{d}} \le \mathbf{C}\,\mathbf{U} \le \overline{\mathbf{d}},
    \label{eq:cert_qp}
\end{equation}
where $\mathbf{U}\in\mathbb{R}^{n}$ stacks the horizon contact forces, $\mathbf{H}\succ 0$ (regularised by the controller), and $(\mathbf{C},\underline{\mathbf{d}},\overline{\mathbf{d}})$ collect the linearised friction pyramid, $f_z$ bounds, and contact-activation constraints. Let $J(\mathbf{U};\mathbf{x}) = \tfrac{1}{2}\mathbf{U}^{\!\top}\mathbf{H}\mathbf{U} + \mathbf{g}^{\!\top}\mathbf{U}$ with constrained optimum $J^\star$. Each cache entry stores
$(\boldsymbol{\phi},\mathbf{x}_c,\mathbf{U}^\star,\boldsymbol{\lambda}_\ell,\boldsymbol{\lambda}_u,J^\star,\mathbf{K})$,
where $(\boldsymbol{\lambda}_\ell,\boldsymbol{\lambda}_u)\!\ge\!\mathbf{0}$ are the dual multipliers of the two sides of \eqref{eq:cert_qp}, and $\mathbf{K} = \partial\mathbf{U}^\star/\partial\mathbf{x}$ is the active-set sensitivity from the entry's KKT system. For the SRBD formulation $(\mathbf{C},\underline{\mathbf{d}},\overline{\mathbf{d}})$ depend on the contact mode but not on the state, so a cached control feasible at insertion stays feasible for every query in the same partition; only suboptimality must be bounded. State-dependent constraint terms restore a non-trivial feasibility component in the certificate below.

\subsection{Per-Query Certificate}
\label{sec:cert}

Given a candidate control $\bar{\mathbf{U}}$ (retrieved, possibly sensitivity-corrected) the certificate returns two quantities.

\paragraph{Primal feasibility.} The exact violation is
\begin{equation}
    \rho_{\mathrm{feas}}(\bar{\mathbf{U}})
    = \max_i \max\!\Big(\underline{d}_i \!-\! (\mathbf{C}\bar{\mathbf{U}})_i,\;
        (\mathbf{C}\bar{\mathbf{U}})_i \!-\! \overline{d}_i,\; 0\Big),
    \label{eq:cert_feas}
\end{equation}
computed in one matrix--vector product.

\paragraph{Dual-gap bound on suboptimality.} The Lagrangian dual of \eqref{eq:cert_qp} at any $\boldsymbol{\lambda}_\ell,\boldsymbol{\lambda}_u\!\ge\!\mathbf{0}$ is, with $\mathbf{s} = \mathbf{g} + \mathbf{C}^{\!\top}(\boldsymbol{\lambda}_u - \boldsymbol{\lambda}_\ell)$,
\begin{equation}
    q(\boldsymbol{\lambda}_\ell,\boldsymbol{\lambda}_u)
    = -\tfrac{1}{2}\mathbf{s}^{\!\top}\mathbf{H}^{-1}\mathbf{s}
      - \boldsymbol{\lambda}_u^{\!\top}\overline{\mathbf{d}}
      + \boldsymbol{\lambda}_\ell^{\!\top}\underline{\mathbf{d}},
    \label{eq:cert_dual}
\end{equation}
and by weak duality $q \le J^\star$. Hence for any feasible $\bar{\mathbf{U}}$,
\begin{multline}
    \Gamma(\bar{\mathbf{U}}) \;=\; J(\bar{\mathbf{U}};\mathbf{x}) - q(\boldsymbol{\lambda}_\ell,\boldsymbol{\lambda}_u) \\
    \;\ge\; J(\bar{\mathbf{U}};\mathbf{x}) - J^\star \;=\; \Delta J(\bar{\mathbf{U}}) \;\ge\; 0
    \label{eq:cert_gap}
\end{multline}
is a rigorous upper bound on the true suboptimality $\Delta J$ of $\bar{\mathbf{U}}$. Throughout the paper $\Delta J$ denotes the true gap and $\Gamma$ the certificate's upper-bound quantity. Evaluating \eqref{eq:cert_dual} costs one Cholesky factorisation of $\mathbf{H}$ and a back-substitution. The reported implementation evaluates $\Gamma$ at $\boldsymbol{\lambda}\!=\!\mathbf{0}$ (the unconstrained relaxation $q = -\tfrac12\mathbf{g}^{\!\top}\mathbf{H}^{-1}\mathbf{g}$), which is the loosest valid weak-duality bound. A cached-dual variant of the bound would be tight whenever the query shared the entry's active set.

\paragraph{Acceptance rule.} Define the per-query relative budget
\begin{equation}
    \beta(\bar{\mathbf{U}}) \;=\; \epsilon_{\mathrm{abs}} + \epsilon_{\mathrm{rel}}\,\lvert J(\bar{\mathbf{U}})\rvert.
    \label{eq:cert_budget}
\end{equation}
The cached candidate is accepted iff
\begin{equation}
    \rho_{\mathrm{feas}}(\bar{\mathbf{U}}) \le \epsilon_{\mathrm{feas}}
    \;\;\text{and}\;\;
    \Gamma(\bar{\mathbf{U}}) \le \beta(\bar{\mathbf{U}}),
    \label{eq:cert_accept}
\end{equation}
otherwise the QP is solved online and the new $(\mathbf{U}^\star,\boldsymbol{\lambda}^\star,J^\star,\mathbf{K})$ tuple is inserted. The budget $\beta$ is state-dependent through $|J(\bar{\mathbf{U}})|$; a per-tick constant guarantee requires an operating tube with an upper bound $J_{\max}$ on $|J|$, giving $\beta\!\le\!\beta_{\max}=\epsilon_{\mathrm{abs}}+\epsilon_{\mathrm{rel}}J_{\max}$.

\paragraph{Dual-feasibility validation.}
The bound \eqref{eq:cert_dual} requires $\boldsymbol{\lambda}_\ell,\boldsymbol{\lambda}_u\!\ge\!\mathbf{0}$ and finite. Because the reported implementation evaluates $\Gamma$ at $\boldsymbol{\lambda}\!=\!\mathbf{0}$, dual-feasibility is satisfied by construction. Any cached-dual variant of the bound would need to guard the reused multipliers for finiteness and non-negativity and fall back to $\boldsymbol{\lambda}\!=\!\mathbf{0}$ on failure, so that the bound the controller acts on is always weak-duality-valid.

\begin{proposition}[Deterministic per-query QP certificate]
\label{prop:cert}
Let \eqref{eq:cert_qp} be the current QP with $\mathbf{H}\!\succ\!\mathbf{0}$ and optimum $J^\star$. Assume $\bar{\mathbf{U}}$ is feasible for \eqref{eq:cert_qp}, i.e.\ $\underline{\mathbf{d}}\le\mathbf{C}\bar{\mathbf{U}}\le\overline{\mathbf{d}}$. Then for any finite, dual-feasible $(\boldsymbol{\lambda}_\ell,\boldsymbol{\lambda}_u)\!\ge\!\mathbf{0}$ the dual-gap quantity $\Gamma(\bar{\mathbf{U}})$ defined in \eqref{eq:cert_gap} satisfies
$$\Delta J(\bar{\mathbf{U}}) \;=\; J(\bar{\mathbf{U}};\mathbf{x}) - J^\star \;\le\; \Gamma(\bar{\mathbf{U}}).$$
If, in addition, the acceptance rule \eqref{eq:cert_accept} succeeds, then $\Delta J(\bar{\mathbf{U}})\le\beta(\bar{\mathbf{U}})$.
\end{proposition}
\begin{proof}
Feasibility of $\bar{\mathbf{U}}$ gives $J^\star \le J(\bar{\mathbf{U}})$; weak duality at any dual-feasible $(\boldsymbol{\lambda}_\ell,\boldsymbol{\lambda}_u)$ gives $q\le J^\star$. Subtracting yields $\Delta J = J(\bar{\mathbf{U}}) - J^\star \le J(\bar{\mathbf{U}}) - q = \Gamma$, and the acceptance rule \eqref{eq:cert_accept} enforces $\Gamma\le\beta$.
\end{proof}

The implementation accepts $\rho_{\mathrm{feas}}\le\epsilon_{\mathrm{feas}}$ with $\epsilon_{\mathrm{feas}}=10^{-4}$, three orders of magnitude below the smallest constraint scale; this is a numerical tolerance for round-off in $\rho_{\mathrm{feas}}$ evaluation rather than a relaxation of feasibility in Proposition~\ref{prop:cert}. Proposition~\ref{prop:cert} also makes no reference to the source of $\bar{\mathbf{U}}$ or $(\boldsymbol{\lambda}_\ell,\boldsymbol{\lambda}_u)$: the bound is valid whether the proposal is obtained by approximate nearest-neighbour search, drawn from an offline dictionary, asynchronously refined, or constructed by sensitivity-corrected affine extrapolation.

\begin{proposition}[Retrieval safety is independent of approximate nearest-neighbour search]
\label{prop:ann}
Suppose the controller applies a cached proposal $\bar{\mathbf{U}}$ only after the acceptance rule \eqref{eq:cert_accept} succeeds with $\rho_{\mathrm{feas}}(\bar{\mathbf{U}})=0$ (exact feasibility), and otherwise solves \eqref{eq:cert_qp} exactly. Then for any retrieval policy --- including one that returns an arbitrary proposal, returns no proposal, or returns one that violates the cached entry's active-set assumptions --- every applied control $\mathbf{U}_{\mathrm{applied}}$ satisfies $\Delta J(\mathbf{U}_{\mathrm{applied}})\le\beta(\mathbf{U}_{\mathrm{applied}})$.
\end{proposition}
\begin{proof}
If the acceptance rule \eqref{eq:cert_accept} succeeds with $\rho_{\mathrm{feas}}(\bar{\mathbf{U}})=0$, Proposition~\ref{prop:cert} gives $\Delta J(\bar{\mathbf{U}})\le\beta(\bar{\mathbf{U}})$; otherwise the controller applies the QP optimum $\mathbf{U}^\star$, for which $\Delta J=0$.
\end{proof}

\noindent The numerical feasibility tolerance $\epsilon_{\mathrm{feas}}$ inherits the same caveat as Proposition~\ref{prop:cert}. Approximate-nearest-neighbour errors therefore reduce the cache hit rate but cannot compromise acceptance correctness; the same statement covers stale refiner entries, sensitivity-corrected proposals whose linearisation has crossed an active-set boundary, and discovered proposals whose anchor was sampled offline.

\subsection{Bounded-Budget Controller Schedule}
\label{sec:bounded_time_contract}

The controller enforces a deterministic schedule at every MPC tick. Let the tick budget $T_{\mathrm{total}}$ be partitioned as $T_{\mathrm{total}} = T_{\mathrm{cache}} + T_{\mathrm{solve}} + T_{\mathrm{fallback}}$:
\begin{enumerate}
\item Within $T_{\mathrm{cache}}$, attempt certified retrieval over the top-$K$ LSH candidates; if any candidate $\bar{\mathbf{U}}$ satisfies \eqref{eq:cert_accept}, apply it and return.
\item Within $T_{\mathrm{solve}}$, solve the QP \eqref{eq:cert_qp}; if the solver returns within budget, apply $\mathbf{U}^\star$ and insert into the cache.
\item If the deadline is missed, apply a shifted last-certified feasible sequence.
\end{enumerate}
The controller executes this schedule at every tick. Section~\ref{sec:bounded_time_demo} reports an empirical stress trial under deliberately tight deadlines. In nominal MPC the analogue of step~(3) is absent: a solver that exceeds its deadline leaves the controller without a defined action.

\subsection{Sensitivity-Augmented Retrieval and Active-Set Filtering}
\label{sec:sensitivity}

A cache entry computed at anchor state $\mathbf{x}_c$ can be upgraded from a constant proposal to a first-order affine proposal via the stored sensitivity:
\begin{equation}
    \bar{\mathbf{U}}(\mathbf{x}) = \mathbf{U}^\star + \mathbf{K}\,(\mathbf{x} - \mathbf{x}_c).
    \label{eq:sens_correct}
\end{equation}
\textbf{Scope.} \eqref{eq:sens_correct} is exact only when the query parameter $\theta = (\mathbf{x},\mathbf{x}_g,\mathbf{r}_p,\boldsymbol{\sigma},\mu,\dots)$ varies along $\mathbf{x}$ alone, with reference $\mathbf{x}_g$, foot positions $\mathbf{r}_p$, contact mode $\boldsymbol{\sigma}$, and model parameters held at the anchor's values. In the present quadruped use-case the actual cache neighbourhood varies along all of these axes simultaneously, so \eqref{eq:sens_correct} is a \emph{first-order proposal heuristic}, not an exact piece of the explicit-MPC law. Proposition~\ref{prop:cert} is what makes its use safe: the certificate is evaluated on the corrected candidate, and the candidate is applied only if the certificate succeeds. Errors in the linearisation (missing $\partial/\partial\mathbf{x}_g$, $\partial/\partial\mathbf{r}_p$, $\partial/\partial\boldsymbol{\sigma}$ terms) reduce hit rate but cannot compromise the bound.

\paragraph{Active-set region filter.} The cached duals $(\boldsymbol{\lambda}_\ell,\boldsymbol{\lambda}_u)$ identify the entry's active set. As a hit-quality pre-screen, the candidate $\bar{\mathbf{U}}$ is rejected when its binding/non-binding pattern is inconsistent with the cached one within a cert-region band of $25$\,N. This filter is not required for Proposition~\ref{prop:cert} and does not represent exact membership in a critical region; it costs one matrix--vector product against $\mathbf{C}$ per query, and avoids the Cholesky in \eqref{eq:cert_dual} for proposals that would fail \eqref{eq:cert_accept} anyway. The region filter is enabled in the production configuration.

\subsection{Certified Online Execution}

Algorithm~\ref{alg:cert_cachempc} summarises the certified retrieval. The condensed-QP data $(\mathbf{H},\mathbf{g},\mathbf{C},\underline{\mathbf{d}},\overline{\mathbf{d}})$ is assembled once per tick because the certificate consumes it; this is the only online cost shared between the hit and miss paths. The hit path additionally pays the certificate cost (one Cholesky of $\mathbf{H}$, one back-substitution, and an $\mathcal{O}(n)$ primal-cost evaluation) and skips the QP solve. An asynchronous refiner re-optimises entries with the largest $\Gamma$ values under a monotone-improvement invariant and runs off the real-time hit/miss path; the per-step guarantee remains the certificate of \eqref{eq:cert_accept}. Cache entries are serialised to disk on shutdown so subsequent sessions start populated.

\begin{algorithm}[htbp]
\caption{Certified CacheMPC Retrieval}
\label{alg:cert_cachempc}
\begin{algorithmic}[1]
\REQUIRE state $\mathbf{x}$, contact mode $\boldsymbol{\sigma}$, cache $\mathcal{H}$, tolerances $(\epsilon_{\mathrm{feas}},\epsilon_{\mathrm{abs}},\epsilon_{\mathrm{rel}})$
\ENSURE force reference $\mathbf{F}_c^{\mathrm{ref}}$
\STATE $\boldsymbol{\phi}\!\leftarrow\!\Phi(\mathbf{x})$;\; select partition $\mathcal{H}_{\boldsymbol{\sigma}}$
\STATE $\mathcal{N}\!\leftarrow\!\textrm{LSH-TopK}(\mathcal{H}_{\boldsymbol{\sigma}},\boldsymbol{\phi},K)$ \COMMENT{up to $K$ neighbour entries, sorted by $\|\boldsymbol{\phi}-\boldsymbol{\phi}_j\|$}
\STATE assemble $(\mathbf{H},\mathbf{g},\mathbf{C},\underline{\mathbf{d}},\overline{\mathbf{d}})$ \COMMENT{condense; no solve}
\STATE $\mathrm{accept}\!\leftarrow\!\textsc{false}$
\FORALL{$(\mathbf{U}^\star_j,\boldsymbol{\lambda}_j,\mathbf{K}_j,\mathbf{x}_{c,j})\in\mathcal{N}$ \textbf{in order, while in $T_{\mathrm{cache}}$ budget}}
    \STATE $\bar{\mathbf{U}}\!\leftarrow\!\mathbf{U}^\star_j + \mathbf{K}_j(\mathbf{x}-\mathbf{x}_{c,j})$ \COMMENT{sensitivity correction}
    \IF{region-filter active set of $\boldsymbol{\lambda}_j$ inconsistent with $\mathbf{C}\bar{\mathbf{U}}$ vs.\ $\underline{\mathbf{d}},\overline{\mathbf{d}}$}
        \STATE \textbf{continue} \COMMENT{region-filter reject; try next candidate}
    \ENDIF
    \STATE $(\rho_{\mathrm{feas}},\Gamma)\!\leftarrow\!\textrm{Certificate}(\mathbf{H},\mathbf{g},\mathbf{C},\underline{\mathbf{d}},\overline{\mathbf{d}},\bar{\mathbf{U}},\boldsymbol{\lambda}\!=\!\mathbf{0})$ \COMMENT{production uses $\boldsymbol{\lambda}\!=\!\mathbf{0}$; cached duals stored for future tighter bound}
    \STATE $\beta\!\leftarrow\!\epsilon_{\mathrm{abs}}+\epsilon_{\mathrm{rel}}|J(\bar{\mathbf{U}})|$
    \IF{$\rho_{\mathrm{feas}}\!\le\!\epsilon_{\mathrm{feas}}$ \AND $\Gamma\!\le\!\beta$}
        \STATE $\mathrm{accept}\!\leftarrow\!\textsc{true}$;\; \textbf{break}
    \ENDIF
\ENDFOR
\IF{$\mathrm{accept}$}
    \STATE $\mathbf{F}_c^{\mathrm{ref}}\!\leftarrow\!\bar{\mathbf{U}}$ \COMMENT{certified hit}
\ELSE
    \STATE $(\mathbf{U}^\star,\boldsymbol{\lambda}^\star,J^\star,\mathbf{K}^\star)\!\leftarrow\!\textrm{SolveMPC}(\mathbf{x},\boldsymbol{\sigma})$
    \STATE insert into $\mathcal{H}_{\boldsymbol{\sigma}}$;\; $\mathbf{F}_c^{\mathrm{ref}}\!\leftarrow\!\mathbf{U}^\star$
\ENDIF
\STATE forward $\mathbf{F}_c^{\mathrm{ref}}$ to the KinoDynamic WBC
\end{algorithmic}
\end{algorithm}

\subsection{Stability Considerations}
\label{sec:stability}

The classical suboptimal-MPC analysis~\cite{scokaert1999suboptimal,pannocchia2011conditions} establishes Lyapunov decrease towards an equilibrium under the standard nominal-MPC assumptions~\cite{mayne2000constrained}; robust extensions under bounded disturbance~\cite{limon2006iss} give input-to-state stability rather than asymptotic convergence. Periodic locomotion violates two of those assumptions: the closed-loop attractor is a phase-parameterised orbit rather than a point, and the applied torque passes through a WBC projection layer that the MPC does not model. A formal practical-orbital stability theorem under the per-query budget $\beta(\bar{\mathbf{U}})$ of \eqref{eq:cert_budget} would require a phase-conditioned value function, recursive feasibility under contact-mode switches, a strong-convexity conversion from cost gap to control deviation (scaling as $\sqrt{\beta/\sigma_{\min}(\mathbf{H})}$ rather than linearly in $\beta$), and a disturbance model that subsumes WBC projection error, SRBD--plant mismatch, and state-estimator error.

Within the present paper the certificate plays an operational role: it bounds the per-query cost gap by $\beta(\bar{\mathbf{U}})$ exactly when \eqref{eq:cert_accept} succeeds, and the proposal mechanism (cache, refiner, sensitivity correction) maximises the rate at which the controller stays inside the certified regime. The KinoDynamic WBC is the final feasibility filter for the full rigid-body dynamics and absorbs part of the residual model mismatch.

\section{Experiments and Results}
\label{sec:experiments}

\subsection{Experimental Setup}

We evaluate CacheMPC on the Unitree Go2 quadruped robot in MuJoCo \cite{mujoco}. The convex predictive backend uses an SRBD model with horizon $N=5$ at $\Delta t = 50$~ms ($T = 0.25$~s); the WBC runs at the proprioceptive estimator rate (250~Hz). Friction coefficient $\mu = 0.3$ and per-foot normal-force bounds $0 \le f_{z,i} \le 150$~N (the lower bound is set to $0$ in the production configuration so that swing feet with $f_{i,z}=0$ remain feasible). The LSH index uses $L=35$ tables of $k=5$ hash functions each, bucket width $w \!\propto\! r$ for an acceptance radius $r=0.20$. The feature vector $\boldsymbol{\phi}\in\mathbb{R}^{18}$ matches \eqref{eq:feature}: $\boldsymbol{\phi} = [\,\dot{\mathbf{p}}_B^\top \;|\; \dot{\mathbf{p}}_{\mathrm{des},B}^\top \;|\; ({}^{\mathcal{B}}\mathbf{r}_{p_1})^\top \;|\; \cdots \;|\; ({}^{\mathcal{B}}\mathbf{r}_{p_4})^\top\,]^\top$, with both velocities and foot positions expressed in the body frame; the desired CoM linear velocity is rotated into the body frame and concatenated with the desired yaw rate. Body-frame bucketing makes the index SE(2)-invariant.

\paragraph{Compute platforms.} Timing in Section~\ref{sec:experiments} is reported on the development workstation (P-W). The on-robot platform (P-O), used in Section~\ref{sec:hardware}, is the NVIDIA Orin NX 16\,GB module shipped on the Go2; the ARM64 Orin NX is approximately $3$--$5\times$ slower than the P-W workstation on the inner QP solve and on Eigen-heavy assembly code. Per-stage breakdowns are wall-clock measurements taken inside the controller covering the entire MPC update call.

We evaluate four cache configurations on each of two downstream architectures. Each configuration is fully specified by three caching flags in the production controller's YAML configuration; Table~\ref{tab:variants} lists the runtime values of \texttt{use\_certificate}, \texttt{use\_sensitivity}, and \texttt{use\_region\_filter}. The region-filter flag defaults to \texttt{true} in the C++ controller, so Cache-Cert retains the region filter even when its YAML override is absent. Convex disables caching entirely.

\begin{table}[htbp]
\centering
\caption{Cache configurations and their three runtime flags.}
\label{tab:variants}
\begin{tabular}{lccc}
\toprule
Configuration & Certificate & Sensitivity & Region filter \\
\midrule
Convex (no cache) & --- & --- & --- \\
Cache-NoCert      & off & off & off \\
Cache-Cert        & on  & off & on  \\
Cache-Full        & on  & on  & on  \\
\bottomrule
\end{tabular}
\end{table}

The downstream architectures are (a) \textbf{WBC}, the KinoDynamic Whole-Body Controller of Section~\ref{sec:wbc} that consumes the MPC's contact-force reference, and (b) \textbf{MRT}, a model-reference-tracking controller that maps the MPC's contact forces to joint torques directly via $\boldsymbol{\tau}=-\mathbf{J}_{aa}^{\top}\mathbf{F}_c$ with a swing-leg foot impedance term, removing the WBC's full-dynamics projection so that the MPC reference reaches joint torque without an absorption layer.

Trials are run on a cold-controller IID protocol: one fresh controller spawn per trial, paired pseudo-random seeds, and a cache file that is wiped at trial start. The simulation evaluation comprises four exploratory campaigns covering the eight (configuration, architecture) cells across velocity-sweep trot, a push ladder from $10$ to $100$\,N, and three terrain scenarios (two stair heights and a $15^{\circ}$ slope). Inspection of the exploratory $n\!=\!15$ data identified three failure-boundary cells with the largest apparent cache-vs-no-cache separation; these cells were prospectively specified and then re-powered to $n\!=\!50$ with fresh seeds in a $600$-trial campaign. Combined with the exploratory campaigns, the full evaluation comprises $2{,}038$ usable cold-controller trials ($1{,}438$ exploratory after two high-push convex trials were lost to a controller startup failure unrelated to caching, plus $600$ re-power). A per-cell anomaly scanner taints trials emitting controller anomalies, NaN/Inf streams, or RMSE values above physically reasonable caps; all cells reported pass these checks. Per-campaign tabulations are provided in the supplementary material.

\paragraph{Hit-rate metrics.} The paper reports three hit-rate quantities, each labelled in context. The \emph{raw LSH-neighbour hit rate} is the fraction of MPC ticks for which the LSH index returns at least one candidate within bucket width $w$. The \emph{certified-acceptance rate} is the fraction for which a returned candidate additionally satisfies the certificate budget \eqref{eq:cert_accept}. The \emph{controller hit-path rate} is the fraction for which the controller takes the cache-hit code path (no QP solve); it equals the certified-acceptance rate when the certificate is on, and the raw LSH-neighbour rate for Cache-NoCert.

\paragraph{Production tolerances.} The condensed Hessian is regularised with $w_{qp}=10^{-2}$, preserving weak duality. The certificate is evaluated with $\boldsymbol{\lambda}\!=\!\mathbf{0}$ throughout. Certificate tolerances are $\epsilon_{\text{abs}}=5.0$, $\epsilon_{\text{rel}}=0.5$, $\epsilon_{\text{feas}}=10^{-4}$, with a cert-region band of $25$\,N. Base proportional gains for $(x,y,\psi)$ are zero because their references are velocity-integrated. Per-trial logged $|J(\bar{\mathbf{U}})|$ values lie in $[10,250]$, giving $\beta\!\in\![10,130]$ across the trot scenario.

\subsection{Solution Retrieval Time}
\label{sec:solve_time}

Per-tick MPC update times on the trot scenario, summarised across $n\!=\!15$ cold-controller trials per cell on the WBC architecture, are reported in Table~\ref{tab:solve_time}. Speedups are per-trial p50 ratios against the matched-seed Convex baseline.

\begin{table}[htbp]
\centering
\caption{Per-tick MPC update time on the trot scenario, $n\!=\!15$ trials per cell. Entries are the across-trial median in microseconds.}
\label{tab:solve_time}
\begin{tabular}{lccccc}
\toprule
Configuration & Mean & p50 & p95 & p99 & Speedup \\
\midrule
Convex       & $759$ & $711$ & $1086$ & $1758$ & ---           \\
Cache-NoCert & $184$ & $\mathbf{28}$  & $1522$ & $2281$ & $\mathbf{25.4\times}$ \\
Cache-Cert   & $842$ & $198$ & $2213$ & $3048$ & $3.6\times$  \\
Cache-Full   & $865$ & $207$ & $2105$ & $2956$ & $3.4\times$  \\
\bottomrule
\end{tabular}
\end{table}

The un-gated cache attains the largest median speedup because its hit path skips both the certificate evaluation and the QP solve, reducing to an LSH lookup followed by the apply step. The cert-gated variants pay the certificate cost (Cholesky on $\mathbf{H}$ plus the dual-gap evaluation) on every tick and plateau at a moderate median speedup; their mean and tail are inflated by certificate-rejected misses that pay both the certificate and the QP solve. The per-trial p50 distributions for the cache variants are non-overlapping with the Convex baseline across the campaign. Hit-rate curves are shown in Fig.~\ref{fig:hit_rate}.

\begin{figure}[htbp]
    \centering
    \includegraphics[width=\linewidth]{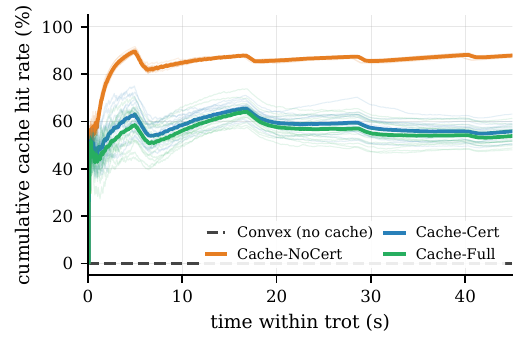}
    \caption{Cumulative cache hit rate over time within trot, $n\!=\!15$ trials. Faint traces are per-trial curves; bold lines are per-variant means. The Cache-NoCert plateau is the raw LSH-neighbour rate; the Cache-Cert and Cache-Full plateaus are the certified-acceptance rate.}
    \label{fig:hit_rate}
\end{figure}

\subsection{Locomotion Performance}\label{sec:tracking}

Tracking performance during the trot velocity sweep ($0$--$0.6$\,m/s) is summarised in Table~\ref{tab:tracking}. Per-trial $v_x$, $v_y$, and combined velocity RMSE were aggregated across $n\!=\!15$ matched-seed trials per configuration.

\begin{table}[htbp]
\centering
\caption{Tracking RMSE on the $0$--$0.6$\,m/s trot velocity sweep, $n\!=\!15$ trials per cell. All RMSE values in m/s.}
\label{tab:tracking}
\begin{tabular}{lcccc}
\toprule
Configuration & $v_x$ RMSE & $v_y$ RMSE & Vel RMSE & Stable \\
\midrule
Convex (no cache) & 0.116 & 0.067 & 0.098 & 15/15 \\
Cache-NoCert      & 0.131 & 0.080 & 0.112 & 15/15 \\
Cache-Cert        & 0.113 & 0.076 & 0.098 & 15/15 \\
Cache-Full        & 0.133 & 0.067 & 0.110 & 15/15 \\
\bottomrule
\end{tabular}
\end{table}

All four cache configurations track within $\pm 14\%$ of the matched-seed Convex baseline on velocity RMSE, the trial-to-trial spread overlaps across cells, and no trial produced a fall. Position RMSE varies widely across cells (bootstrap CI $0.6$--$1.2$\,m) because the base-frame proportional gains for $(x,y,\psi)$ are zero in the production configuration and position drifts under velocity-integrated references; velocity RMSE is therefore the discriminating tracking metric.

\subsection{Failure-Boundary Safety at $n\!=\!50$}
\label{sec:safety_boundary}

Closed-loop safety is characterised by whether the cache changes the stable rate at conditions that push the no-cache baseline to fail. The exploratory campaigns bracket two failure boundaries (WBC architecture between $80$ and $100$\,N of lateral push; MRT architecture between $50$ and $60$\,N, with an additional terrain failure at the $10$\,cm step height). The three cells with the largest apparent cache-vs-no-cache separation in the exploratory data were prospectively specified and re-powered to $n\!=\!50$ with fresh seeds; Table~\ref{tab:boundary_safety} reports the result.

\begin{table}[htbp]
\centering
\caption{Failure-boundary stable rate per configuration, $n\!=\!50$ trials per cell. Fisher $p$-values are against the no-cache baseline of the same cell and are uncorrected; the Holm--Bonferroni threshold over the 12 comparisons is $\alpha/12\!\approx\!0.0042$, so the bold entry is not corrected-significant.}
\label{tab:boundary_safety}
\begin{tabular}{llccc}
\toprule
Cell & Configuration & Stable & Rate & Fisher $p$ \\
\midrule
\multicolumn{5}{l}{WBC, $80$\,N lateral push} \\
 & Convex (no cache) & 19/50 & 38\%          & ---   \\
 & Cache-NoCert      & 28/50 & 56\%          & 0.109 \\
 & Cache-Cert        & 24/50 & 48\%          & 0.419 \\
 & Cache-Full        & 31/50 & \textbf{62\%} & \textbf{0.027} \\
\midrule
\multicolumn{5}{l}{MRT, $50$\,N lateral push} \\
 & Convex            & 10/50 & 20\%          & ---   \\
 & Cache-NoCert      &  9/50 & 18\%          & 1.000 \\
 & Cache-Cert        &  7/50 & 14\%          & 0.595 \\
 & Cache-Full        &  6/50 & 12\%          & 0.414 \\
\midrule
\multicolumn{5}{l}{MRT, $10$\,cm step traversal} \\
 & Convex            & 43/50 & 86\%          & ---   \\
 & Cache-NoCert      & 44/50 & 88\%          & 1.000 \\
 & Cache-Cert        & 43/50 & 86\%          & 1.000 \\
 & Cache-Full        & 46/50 & 92\%          & 0.525 \\
\bottomrule
\end{tabular}
\end{table}

\begin{figure}[htbp]
    \centering
    \includegraphics[width=\linewidth]{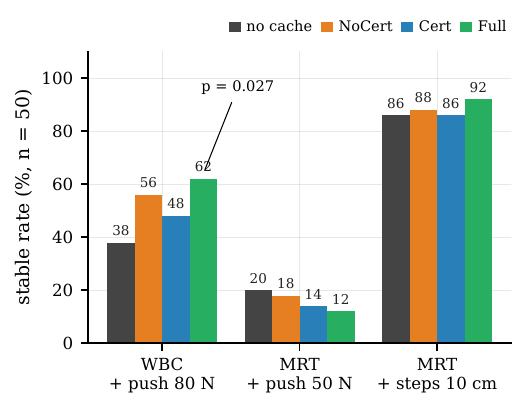}
    \caption{Failure-boundary stable rate, $n\!=\!50$ trials per cell, at three prospectively specified cells. The annotated $p\!=\!0.027$ (Cache-Full vs.\ Convex on the WBC $80$\,N cell) is the only within-cell comparison below $0.05$ and is uncorrected; it does not survive Holm--Bonferroni over the 12-comparison family.}
    \label{fig:boundary_safety}
\end{figure}

No comparison between any cache variant and the no-cache baseline survives Holm--Bonferroni correction over the 12 cell--variant comparisons at $n\!=\!50$. The largest within-cell separation in the table is descriptive only.

The operational statement supported by the data is that caching without a per-query gate does not measurably harm closed-loop stability at the failure boundaries tested. The certificate's runtime safety contribution is masked at this sample size by the downstream architecture --- the WBC's full-dynamics projection and joint-level proportional-derivative loop on one side, the foot-impedance term on the MRT side. At the matched $10$\,N exploratory cell, all four configurations survive every trial and the per-trial $v_y$ RMSE distributions overlap.

\subsection{Bounded-Budget Stress Test}
\label{sec:bounded_time_demo}

The schedule of Section~\ref{sec:bounded_time_contract} was exercised under deliberately tight deadlines on a single $45$\,s trot trial. Table~\ref{tab:stress_test} reports the outcome.

\begin{table}[htbp]
\centering
\caption{Stress trial under tight deadlines on $45$\,s of trot. Schedule parameters: $K\!=\!3$, $T_{\mathrm{cache}}\!=\!200\,\mu$s, $T_{\mathrm{solve}}\!=\!2000\,\mu$s. The closed-loop velocity RMSE is reported against a matched-seed unstressed trial.}
\label{tab:stress_test}
\begin{tabular}{lc}
\toprule
Quantity & Value \\
\midrule
Controller ticks (TROT) & $2239$ \\
Cache-budget exhausts (fall to QP solve) & $469$ \\
Solver-deadline overruns (fall to shifted fallback) & $56$ \\
Total deadline-fallback ticks & $525$ ($23.4\%$) \\
Vel RMSE (this trial) & $0.079$\,m/s \\
Vel RMSE (matched unstressed trial) & $0.078$\,m/s \\
Fall? & no \\
\bottomrule
\end{tabular}
\end{table}

The schedule meets its wall-clock budget on every tick of the trial without a fall. This is descriptive single-trial evidence of behaviour under timing pressure, not a statistical comparison.

\section{First Hardware Deployment}
\label{sec:hardware}

A first-deploy session was conducted on the Unitree Go2 with the on-robot NVIDIA Orin NX 16\,GB module (P-O); the workstation used for the simulation campaign is denoted P-W. The cross-compiled controller binary follows the same code path as the simulation controller, with the hardware build enabling Unitree DDS communication and the on-robot proprioceptive estimator. The cache file is seeded from the simulation-trained Cache-Full cache obtained at the end of the WBC trot campaign. The session is an informal free-walk: the operator drives the robot through approximately $90$\,s of trot per configuration via the gamepad and applies $20$--$30$ lateral pushes by hand. Push events are detected from $|v_y|\!>\!0.3$\,m/s during trot.

The scope of this session is limited: timing characterisation, qualitative push response, and confirmation that the simulation-trained cache transfers. The pushes are operator-applied and not controlled for force, direction, or timing across configurations, so the section reports qualitative free-walk evidence rather than a controlled comparative safety result.

Table~\ref{tab:hw_summary} reports the headline numbers; Fig.~\ref{fig:hw_cdf} plots the per-tick update-time CDF.

\begin{table}[htbp]
\centering
\caption{Hardware results on the Orin NX. The hit-rate row reports the raw LSH-neighbour rate for Cache-NoCert and the certified-acceptance rate for Cache-Full. Timing entries are in $\mu$s; push events use $|v_y|\!>\!0.3$\,m/s during trot.}
\label{tab:hw_summary}
\begin{tabular}{lccc}
\toprule
Metric & Convex & Cache-NoCert & Cache-Full \\
\midrule
Hit rate          & --- & $55.8\%$ & $16.0\%$ \\
Mean              & $1966$ & $1324$ & $2910$ \\
p50               & $1883$ & $\mathbf{101}$ & $3342$ \\
p95               & $2478$ & $3753$ & $4267$ \\
p99               & $3856$ & $5139$ & $5741$ \\
\midrule
Push events       & $21$    & $32$    & $20$ \\
Peak $|v_y|$ (m/s)& $0.84$  & $0.83$  & $0.68$ \\
Min $p_z$ (m)     & $0.301$ & $0.296$ & $0.282$ \\
Falls             & $0$     & $0$     & $0$ \\
\bottomrule
\end{tabular}
\end{table}

\begin{figure}[htbp]
    \centering
    \includegraphics[width=\linewidth]{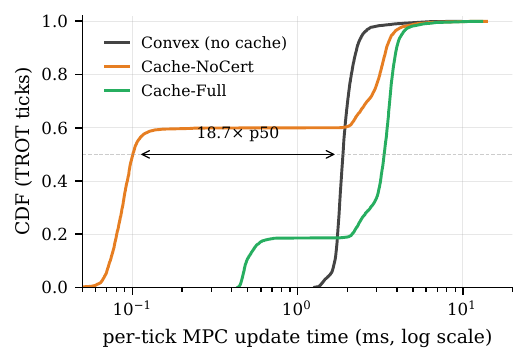}
    \caption{Per-tick MPC update-time CDF on the Orin NX, trot ticks only. The $18.7\times$ p50 Cache-NoCert speedup over Convex is visible at the $0.5$ line; Cache-Full is bimodal with a certified-acceptance plateau near $0.5$\,ms and a climb at $\sim\!3.5$\,ms from certificate-rejected ticks.}
    \label{fig:hw_cdf}
\end{figure}

The un-gated cache reaches the predicted on-robot p50 timing and shows no fall across the push events. The simulation-trained cache transfers: certified acceptance is non-zero on the first trot phase. Cache-Full is the inverted case relative to simulation, slower than Convex on the Orin NX, and its mechanism is summarised in Table~\ref{tab:hw_cachefull}.

\begin{table}[htbp]
\centering
\caption{Cache-Full on the Orin NX: the certificate's high rejection rate explains the inversion of the simulation-measured speedup. The raw LSH-neighbour rate is taken from the matched Cache-NoCert run as a cross-run proxy; the derived rejection rate assumes equal raw-LSH rates across runs.}
\label{tab:hw_cachefull}
\begin{tabular}{lc}
\toprule
Quantity & Value \\
\midrule
Raw LSH-neighbour rate (proxy, Cache-NoCert) & $55.8\%$ \\
Certified-acceptance rate (Cache-Full) & $16.0\%$ \\
Derived certificate rejection rate & $\approx\!71\%$ \\
Cache-Full p50 vs Convex p50 (sim) & $\mathbf{3.4\times}$ faster \\
Cache-Full p50 vs Convex p50 (hardware) & $1.8\times$ slower \\
\bottomrule
\end{tabular}
\end{table}

The hardware inversion is consistent with the certificate of Proposition~\ref{prop:cert} acting as designed: $\Gamma$ inflates above $\beta(\bar{\mathbf{U}})$ for the majority of LSH-neighbour candidates under the on-robot QP data, so most candidates are rejected and fall through to the full QP solve. Certificate-rejected ticks pay both the certificate evaluation and the QP solve, putting Cache-Full systematically above Convex.

\section{Threats to Validity}
\label{sec:threats}

\paragraph{Single platform and dynamics setting.} The simulation campaign uses friction $\mu=0.3$, the default MuJoCo contact solver, and flat or scripted-terrain ground. Within these conditions it covers two architectures (WBC, MRT), four cache configurations, a push ladder from $10$ to $100$\,N, and two terrain types. Generalisation across payload, contact-model variation, friction coefficient, and quadruped platform is not characterised. The hardware session covers a single platform (Unitree Go2 + Orin NX) under operator-driven free-walk on flat ground.

\paragraph{Loose certificate bound.} The certificate is evaluated in production with $\boldsymbol{\lambda}\!=\!\mathbf{0}$, the unconstrained relaxation; this gives the loosest valid weak-duality bound on $\Delta J$. A cached-dual variant of the bound, which would be tight whenever the query shares the entry's active set, is not used in the present evaluation.

\paragraph{Single push direction.} The push tests use a single lateral direction, both at the simulation failure boundary and on hardware. A $(\pm x,\pm y)$ sweep at $n\!=\!50$ would characterise the geometry of the boundary more fully.

\paragraph{Sample-size-bounded negative result.} The $n\!=\!50$ Fisher $p$-values reported in Table~\ref{tab:boundary_safety} are uncorrected and do not establish equivalence: ``no significant difference detected'' is a statement about test power, not about the magnitude of the underlying effect. The certificate may carry measurable closed-loop value in regimes that the present evaluation cannot resolve.

\paragraph{Unit of analysis.} The timing distributions in Table~\ref{tab:solve_time} are aggregated within trial and summarised across trials; per-tick hypothesis testing is avoided because per-tick samples within a trial are autocorrelated. The per-trial p50 ratios are descriptive across $n\!=\!15$ independent cold-controller spawns per cell.

\paragraph{Baseline coverage.} The cache variants are compared only against the no-cache Convex baseline of the same MPC implementation. Warm-started convex MPC, kd-tree nearest-neighbour retrieval~\cite{bentley1975kdtree}, fixed offline dictionaries, and learned approximate-MPC policies~\cite{mpcnet,hertneck2018lcss} are not included in the present comparison.

\paragraph{Reproducibility protocol.} Every cold-controller trial records the binary SHA-256, the YAML configuration that ran, the per-trial cache file, the per-tick controller log, and the scenario time-series log. Per-cell trial logs are gated by an anomaly scanner with pre-specified rejection criteria (controller anomalies, NaN/Inf streams, RMSE values above physically reasonable caps); all $n\!=\!50$ re-power cells and all reported exploratory cells pass these checks. The same controller binary path is used in simulation and on hardware; only the communication and estimator backends differ. Aggregate per-cell tabulations, per-knob sweep tables, and the full hardware-session breakdown are available from the authors on reasonable request.

\section{Conclusion}
\label{sec:conclusion}

This paper presented Certified CacheMPC, a memory-augmented MPC framework in which an LSH-indexed cache returns previously computed contact-force trajectories and an a-posteriori per-query certificate produces a Lagrangian dual-gap quantity $\Gamma(\bar{\mathbf{U}})$ that upper-bounds the true suboptimality $\Delta J$ of any feasible accepted retrieval. The certificate is deterministic and source-independent. A bounded-budget controller schedule combines top-$K$ certified retrieval, a deadline-bounded QP solve, and a shifted last-certified fallback.

The framework was evaluated through a cold-controller simulation campaign and a first-deploy session on the on-robot NVIDIA Orin NX; the headline numbers are in Tables~\ref{tab:solve_time}--\ref{tab:hw_cachefull}. The un-gated cache produces a large median solve-time speedup in simulation that transfers to hardware. No statistically significant difference between the cache variants and the no-cache baseline is detected at $n\!=\!50$ at any tested failure-boundary cell. The certificate functions as a verifiable upper bound on the cost gap of any applied control, rather than as a runtime safety mechanism whose effect the present evaluation can resolve.

\section{Future Work}
\label{sec:future_work}

The directions opened by this work are grouped into four families. \emph{Tighter certified bounds.} Reuse of cached multipliers $(\boldsymbol{\lambda}_\ell,\boldsymbol{\lambda}_u)$ would tighten $\Gamma$ towards $\Delta J$ whenever the query shares the entry's active set, given a guarded staleness check; a Hoffman-bound perturbation analysis would extend Proposition~\ref{prop:cert} to $\rho_{\mathrm{feas}}>0$. \emph{Formal closed-loop theory.} A practical-orbital stability theorem on top of the per-query budget $\beta(\bar{\mathbf{U}})$ and a worst-case real-time theorem on top of the bounded-budget schedule of Section~\ref{sec:bounded_time_contract} are natural next steps, requiring a phase-conditioned value function, recursive feasibility under contact-mode switches, and a preemption-safe inner solver. \emph{Empirical breadth.} A constant-velocity hardware tracking matrix, a powered non-inferiority hardware safety matrix, $(\pm x,\pm y)$ push sweeps at $n\!=\!50$, and richer terrain (steps, slopes, granular ground) are natural next campaigns. \emph{Baselines and cache lifecycle.} Warm-started convex MPC, kd-tree neighbour search, fixed offline dictionaries, and learned approximate-MPC policies provide stronger reference points; on-robot fine-tuning and online refinement of the cache are expected to reduce the certificate rejection rate observed on hardware.

\end{document}